\documentclass{article}
\usepackage{arxiv}
\renewcommand{\headeright}{}
\renewcommand{\undertitle}{}
\renewcommand{\shorttitle}{}
\usepackage[utf8]{inputenc}
\usepackage[T1]{fontenc}
\usepackage{hyperref}
\usepackage{url}
\usepackage{booktabs}
\usepackage{amsfonts}
\usepackage{amsmath,amssymb,amsthm}
\usepackage{nicefrac}
\usepackage[final,babel=false]{microtype}
\usepackage{graphicx}
\usepackage{tikz}
\usepackage{enumitem}
\setlist[itemize]{itemsep=4pt, topsep=6pt, parsep=2pt}
\setlist[enumerate]{itemsep=4pt, topsep=6pt, parsep=2pt}
\AtBeginDocument{%
  \setlength{\abovedisplayskip}{10pt plus 3pt minus 3pt}%
  \setlength{\belowdisplayskip}{10pt plus 3pt minus 3pt}%
  \setlength{\abovedisplayshortskip}{6pt plus 2pt}%
  \setlength{\belowdisplayshortskip}{6pt plus 2pt}%
}
\usepackage{etoolbox}
\makeatletter
\patchcmd{\@begintheorem}{\trivlist}{\setlength{\topsep}{8pt}\trivlist}{}{}
\makeatother
\AtBeginEnvironment{definition}{\vspace{2pt}}
\AtEndEnvironment{definition}{\vspace{2pt}}
\AtBeginEnvironment{assumption}{\vspace{2pt}}
\AtEndEnvironment{assumption}{\vspace{2pt}}
\AtBeginEnvironment{theorem}{\vspace{2pt}}
\AtEndEnvironment{theorem}{\vspace{2pt}}
\AtBeginEnvironment{lemma}{\vspace{2pt}}
\AtEndEnvironment{lemma}{\vspace{2pt}}
\AtBeginEnvironment{proposition}{\vspace{2pt}}
\AtEndEnvironment{proposition}{\vspace{2pt}}
\AtBeginEnvironment{corollary}{\vspace{2pt}}
\AtEndEnvironment{corollary}{\vspace{2pt}}
\AtBeginEnvironment{remark}{\vspace{2pt}}
\AtEndEnvironment{remark}{\vspace{2pt}}
\AtBeginEnvironment{example}{\vspace{2pt}}
\AtEndEnvironment{example}{\vspace{2pt}}
\renewcommand{\paragraph}[1]{\medskip\noindent\textbf{#1}\enspace}
\usepackage{fancyhdr}
\fancypagestyle{plain}{%
  \fancyhf{}%
  \fancyfoot[C]{\thepage}%
}
\fancypagestyle{fancy}{%
  \fancyhf{}%
  \fancyfoot[C]{\thepage}%
}
\AtBeginDocument{\thispagestyle{plain}}

\newtheorem{theorem}{Theorem}[section]

\newtheorem{proposition}[theorem]{Proposition}
\newtheorem{corollary}[theorem]{Corollary}
\newtheorem{definition}[theorem]{Definition}

\renewcommand{\qed}{\hfill$\square$}

\title{Quantifying Automation Risk in High-Automation AI Systems:\\
A Bayesian Framework for Failure Propagation and Optimal Oversight}

\author{
  Vishal Srivastava \\
  Whiting School of Engineering \\
  Johns Hopkins University \\
  Baltimore, MD 21218 \\
  \texttt{vsrivas7@jhu.edu} \\
  \And
  Tanmay Sah \\
  Harrisburg University of Science and Technology \\
  Harrisburg, PA \\
  \texttt{TSah@my.harrisburgu.edu} \\
}
\date{}

\begin{document}

\maketitle

\begin{abstract}
\noindent
Organizations across finance, healthcare, transportation, content moderation, and critical infrastructure are rapidly deploying highly automated AI systems, yet they lack principled methods to quantify how increasing automation amplifies harm when failures occur. We propose a parsimonious Bayesian risk decomposition expressing expected loss as:
\[
\mathbb{E}[\text{Loss}] = P(F) \times P(H \mid F, A) \times \mathbb{E}[S \mid H],
\]
where $F$ denotes system failure, $H$ denotes harm, $S$ denotes severity, and $A \in [0,1]$ represents automation level. This framework isolates a critical quantity--$P(H \mid F, A)$, the conditional probability that failures propagate into harm--which captures execution and oversight risk rather than model accuracy alone. We develop complete theoretical foundations: formal proofs of the decomposition, a harm propagation equivalence theorem linking $P(H \mid F, A)$ to observable execution controls, risk elasticity measures, efficient frontier analysis for automation policy, and optimal resource allocation principles with second-order conditions. We motivate the framework with an illustrative case study of the 2012 Knight Capital incident (\$440M loss) as one instantiation of a broadly applicable failure pattern, and characterize the research design required to empirically validate the framework at scale across deployment domains. This work provides the theoretical foundations for a new class of deployment-focused risk governance tools for agentic and automated AI systems.
\end{abstract}


\section{Introduction}

The rapid adoption of artificial intelligence across sectors has created a critical governance challenge: balancing the efficiency gains of automation against the risks of uncontrolled failure propagation. High-profile incidents reveal a common pattern across domains--harm often arises not from model inaccuracy \emph{per se}, but from failures executing automatically without sufficient controls. The 2012 Knight Capital trading algorithm failure (\$440 million loss in 45 minutes) \cite{knight2013} illustrates this pattern starkly in finance; analogous dynamics appear in automated diagnostic support systems that act without clinician review \cite{Obermeyer2019}, in autonomous vehicle control stacks where sensor failures propagate to actuation without fallback \cite{Koopman2017}, in content moderation pipelines that remove or amplify content at scale without human checkpoints \cite{Gillespie2018}, and in critical infrastructure control systems where automated responses to anomalous sensor readings can cascade into outages \cite{Lee2017}. What these incidents share is not primarily a modelling failure, but a deployment failure: the architecture allowed erroneous outputs to reach consequential action with insufficient opportunity for detection or intervention.

Existing risk frameworks for AI systems focus predominantly on reducing model failure probability $P(F)$--improving accuracy, robustness, and validation. They largely neglect a complementary lever: the probability that a failure, once it occurs, propagates into harm given the level of automation $P(H \mid F, A)$. This gap becomes increasingly consequential as AI systems become more agentic, operating with greater autonomy and fewer human checkpoints.

This paper addresses this gap by developing a complete theoretical framework for decomposing, analyzing, and optimizing automation risk. The framework is domain-agnostic: it applies wherever an AI system can take consequential actions at a level of automation that may outpace human oversight. Our approach separates the problem into three orthogonal components--technical risk, deployment risk, and consequence risk--and provides formal tools for reasoning about each.

\subsection{Contributions}

We make four contributions. First, we prove the expected loss decomposition $P(F) \times P(H|F,A) \times \mathbb{E}[S|H]$ from first principles (Section~\ref{sec:theory}), establishing it as a Bayesian identity under clearly stated assumptions rather than an ad hoc heuristic. Second, we establish a harm propagation equivalence (Theorem~\ref{thm:harm_prop}) showing that $P(H|F,A)$ equals the probability that a failed output is executed, linking an unobservable risk quantity to observable system design characteristics applicable across deployment contexts. Third, we derive risk elasticity measures, characterize cost-minimizing automation levels, establish second-order sufficiency conditions, and define the efficient frontier of Pareto-optimal automation policies (Section~\ref{sec:theory}). Fourth, we illustrate the framework's applied value through the Knight Capital incident as a representative high-automation deployment failure, and characterize the identification strategies required to empirically estimate causal automation gradients from observational data across domains (Sections~\ref{sec:casestudy}--\ref{sec:empirical}).

\section{Related Work}

\subsection{Safety Engineering and Risk Decomposition}

Fault tree analysis \cite{vesely1981} and STAMP \cite{leveson2004} decompose accident pathways into technical failures and control loop breakdowns. Our $P(H|F,A)$ operationalizes control effectiveness for AI systems, paralleling barrier failure probabilities in reliability engineering. Ericson \cite{ericson2005} emphasizes layered defenses; our automation dimensions map to defensive depth. Systems Theoretic Accident Model and Processes (STAMP) is particularly relevant for agentic AI: both treat accidents as emergent from control structure inadequacies rather than isolated component failures.

\subsection{AI Governance Frameworks}

The NIST AI Risk Management Framework \cite{NIST2023} structures AI risk across govern, map, measure, and manage functions but does not provide a formal decomposition of how automation level modulates harm probability. Our framework complements the NIST RMF by supplying the quantitative backbone for the ``measure'' function: $P(H|F,A)$ provides a deployment-focused metric that sits alongside model performance metrics. The EU AI Act \cite{EU2021} mandates human oversight for high-risk AI systems across sectors including healthcare, critical infrastructure, employment, education, and law enforcement; our framework provides quantitative implementation guidance by specifying what it means to reduce $P(H|F,A)$ and by how much, enabling organizations to calibrate oversight investments proportionately to risk. ISO/IEC 42001 \cite{ISO42001} on AI management systems similarly lacks a formal link between automation configuration and expected harm; the decomposition developed here could underpin conformance assessment metrics.

\subsection{Human-Automation Interaction}

Foundational work by Parasuraman et al.\ \cite{Parasuraman2000} and Sheridan \cite{Sheridan1992} establishes taxonomies of automation levels and documents automation bias effects across domains including aviation, medicine, and process control. Our framework operationalizes these concepts with formal optimality conditions and an efficient frontier characterization, providing a decision-theoretic grounding for automation level choices.

\subsection{Causal Inference in Observational Studies}

Rosenbaum \cite{Rosenbaum2002} develops sensitivity analysis for unmeasured confounding; Imbens \& Rubin \cite{Imbens2015} emphasize propensity scores and instrumental variables; Manski \cite{manski1990} provides partial identification bounds under selection. These methods are directly applicable to validating our framework empirically across incident databases, as described in Section~\ref{sec:empirical}.

\section{Theoretical Framework}
\label{sec:theory}

\subsection{Foundational Risk Decomposition}

\begin{theorem}[Expected Loss Decomposition]
\label{thm:loss_decomp}
Let $F$ denote system failure, $H$ denote harm, $A \in [0,1]$ denote automation level, and $S$ denote loss severity. Under the assumptions that (i) harm occurs only when failures propagate to execution, and (ii) automation level is a deployment choice independent of system parameters affecting failure probability ($P(F|A) = P(F)$), the expected loss per decision is:
\begin{equation}
\mathbb{E}[\text{Loss}] = P(F) \times P(H \mid F, A) \times \mathbb{E}[S \mid H].
\label{eq:main}
\end{equation}
\end{theorem}

\begin{proof}
By the law of total expectation,
\[
\mathbb{E}[\text{Loss}] = \mathbb{E}[\text{Loss} \mid H]\,P(H) + \mathbb{E}[\text{Loss} \mid \neg H]\,P(\neg H).
\]
Assuming zero loss when no harm occurs ($\mathbb{E}[\text{Loss} \mid \neg H] = 0$) and defining $\mathbb{E}[S \mid H] = \mathbb{E}[\text{Loss} \mid H]$:
\[
\mathbb{E}[\text{Loss}] = \mathbb{E}[S \mid H] \cdot P(H).
\]
Decompose $P(H)$ via conditioning on failure $F$ and automation $A$:
\[
P(H) = \sum_{f \in \{0,1\}} \int_0^1 P(H \mid F=f, A=a)\,P(F=f \mid A=a)\,p(A=a)\,da.
\]
By assumption (ii), $P(F \mid A) = P(F)$. For a deployed system with fixed $A$:
\[
P(H) = P(H \mid F=0, A)\,P(F=0) + P(H \mid F=1, A)\,P(F=1).
\]
By assumption (i), harm requires failure: $P(H \mid F=0, A) = 0$. Therefore:
\[
P(H) = P(F) \cdot P(H \mid F, A).
\]
Substituting into the expected loss expression yields equation~\eqref{eq:main}. \qed
\end{proof}

\textbf{Interpretation.} The decomposition separates three distinct risk components: (1) \emph{technical risk} $P(F)$, representing model accuracy and robustness; (2) \emph{deployment risk} $P(H \mid F, A)$, representing execution controls and oversight; and (3) \emph{consequence risk} $\mathbb{E}[S \mid H]$, representing the stakes of the decision domain. Critically, organizations can invest in either $P(F)$ or $P(H \mid F, A)$ reduction independently--this clean separation is the primary analytical value of the framework. The decomposition is domain-agnostic: $S$ may represent financial loss, patient harm, traffic fatalities, civil liberties violations, or infrastructure downtime, depending on the application context.

\textbf{Remark on Assumption (ii).} The independence $P(F|A)=P(F)$ holds when model architecture and validation are fixed independently of deployment configuration--a common organisational structure where model development and deployment teams operate separately. It may be violated if, for example, high-automation deployments attract greater testing rigor (downward bias on gradient estimates) or if automated systems encounter higher decision volumes that surface additional edge cases (upward bias). Section~\ref{sec:casestudy} discusses implications for the Knight Capital case; analogous considerations apply in other domains.

\subsection{Harm Propagation as Execution Probability}

\begin{theorem}[Harm Propagation Equivalence]
\label{thm:harm_prop}
Let $U$ denote the event that a failed output is executed. Under the assumptions:
\begin{enumerate}
    \item Harm requires execution: $H \Rightarrow U$,
    \item Execution implies failure consideration: $U \Rightarrow F$,
    \item Harm is certain given harmful execution: $P(H \mid U, F, A) = 1$,
\end{enumerate}
then $P(H \mid F, A) = P(U \mid F, A)$.
\end{theorem}

\begin{proof}
By assumption (1), $H \Rightarrow U$, so $H \cap \neg U = \emptyset$:
\[
P(H \mid F, A) = P(H \cap U \mid F, A) + P(H \cap \neg U \mid F, A) = P(H \cap U \mid F, A).
\]
By the chain rule:
\[
P(H \cap U \mid F, A) = P(H \mid U, F, A)\,P(U \mid F, A).
\]
From assumption (3), $P(H \mid U, F, A) = 1$, yielding $P(H \mid F, A) = P(U \mid F, A)$. \qed
\end{proof}

\textbf{Operational Interpretation.} This theorem is practically significant: it re-expresses an unobservable quantity (conditional harm probability) as the probability that a failure executes--something that can be directly designed and measured through kill-switch latency, override capability, and detection speed. These control levers are identifiable and instrumentable across deployment domains: a clinical decision support system can be designed with mandatory clinician confirmation; an autonomous vehicle stack can require human takeover above a confidence threshold; a content moderation pipeline can route borderline cases to human review queues.

\subsection{Automation Parameterization}

\begin{definition}[Automation Dimensions]
\label{def:automation_dims}
For an AI system operating in a consequential domain, we define three automation dimensions:
\begin{itemize}
    \item $A_{\text{decision}} \in \{0, 0.5, 1\}$: Decision authority (human decides / human approves / system decides)
    \item $A_{\text{override}} \in [0,1]$: Inverse of human override capability
    \item $A_{\text{detection}} \in [0,1]$: Inverse of failure detection speed
\end{itemize}
\end{definition}

\begin{definition}[Automation Aggregation]
\label{def:automation_agg}
The primary aggregation uses the maximum operator, implementing a \emph{weakest-link} principle:
\begin{equation}
A = \max(A_{\text{decision}},\, A_{\text{override}},\, A_{\text{detection}}).
\label{eq:automation_max}
\end{equation}
\end{definition}

\begin{proposition}[Monotonicity]
\label{prop:monotonicity}
If $A_{\text{decision}} \leq A'_{\text{decision}}$, $A_{\text{override}} \leq A'_{\text{override}}$, and $A_{\text{detection}} \leq A'_{\text{detection}}$, then $A \leq A'$.
\end{proposition}

\begin{proof}
Immediate from the monotonicity of the maximum function. \qed
\end{proof}

\textbf{Remark on aggregation.} Alternatives include the arithmetic mean (weighting dimensions equally) or the product (requiring all dimensions to be low for low overall automation). The max operator is appropriate when any single dimension can create a pathway to unchecked harm--consistent with layered defense principles in safety engineering \cite{ericson2005}. Sensitivity to aggregation choice should be evaluated empirically; see Section~\ref{sec:empirical}.

\subsection{Marginal Effects and Risk Elasticity}

\begin{theorem}[Automation Gradient]
\label{thm:gradient}
The marginal effect of automation on expected loss is:
\begin{equation}
\frac{\partial \mathbb{E}[\text{Loss}]}{\partial A} = P(F) \cdot \frac{\partial P(H \mid F, A)}{\partial A} \cdot \mathbb{E}[S \mid H].
\label{eq:gradient}
\end{equation}
\end{theorem}

\begin{proof}
Differentiate equation~\eqref{eq:main} with respect to $A$, treating $P(F)$ and $\mathbb{E}[S|H]$ as independent of $A$ under assumption (ii):
\[
\frac{\partial \mathbb{E}[\text{Loss}]}{\partial A} = P(F) \cdot \mathbb{E}[S \mid H] \cdot \frac{\partial P(H \mid F, A)}{\partial A}.
\]
\end{proof}

\begin{corollary}[Risk Elasticity]
\label{cor:elasticity}
The elasticity of expected loss with respect to automation is:
\begin{equation}
\varepsilon_{A} = \frac{\partial \log \mathbb{E}[\text{Loss}]}{\partial \log A} = \frac{\partial \log P(H \mid F, A)}{\partial \log A}.
\label{eq:elasticity}
\end{equation}
For small changes $\Delta A$:
\[
\frac{\Delta \mathbb{E}[\text{Loss}]}{\mathbb{E}[\text{Loss}]} \approx \varepsilon_A \cdot \frac{\Delta A}{A}.
\]
\end{corollary}

\begin{proof}
From Theorem~\ref{thm:gradient}:
\[
\varepsilon_A = \frac{\partial \mathbb{E}[\text{Loss}]}{\partial A} \cdot \frac{A}{\mathbb{E}[\text{Loss}]} = \frac{P(F)\,\mathbb{E}[S|H]\,\frac{\partial P(H|F,A)}{\partial A}\cdot A}{P(F)\,P(H|F,A)\,\mathbb{E}[S|H]} = \frac{\partial P(H|F,A)}{\partial A}\cdot\frac{A}{P(H|F,A)}.
\]
Recognizing $\frac{\partial \log P(H|F,A)}{\partial \log A} = \frac{\partial P(H|F,A)}{\partial A}\cdot\frac{A}{P(H|F,A)}$ completes the derivation. \qed
\end{proof}

\textbf{Interpretation.} The elasticity $\varepsilon_A$ is the fundamental governance parameter: a 1\% increase in automation level induces an $\varepsilon_A$\% increase in expected loss, entirely through the harm propagation channel. This parameter is domain-specific--one would expect $\varepsilon_A$ to be higher in domains where failures are irreversible (e.g., autonomous vehicle actuation, surgical robotics) than in domains where errors can be corrected post-hoc (e.g., content recommendation). Empirical estimation of $\varepsilon_A$ across deployment domains is the key target for future quantification work.

\subsection{Optimal Resource Allocation}

\begin{proposition}[Validation Budget Optimization]
\label{prop:optimal_allocation}
Let $B$ be the total validation budget, $c_F$ and $c_A$ the marginal costs of reducing $P(F)$ and $P(H \mid F, A)$, and $x_F$, $x_A$ the resources allocated to each. Under diminishing returns--$P(F) = f(x_F)$, $P(H \mid F, A) = g(x_A)$ with $f', g' < 0$--optimal allocation satisfies:
\begin{equation}
\frac{P(H \mid F, A) \cdot \mathbb{E}[S \mid H]}{c_F} \cdot |f'(x_F^*)| = \frac{P(F) \cdot \mathbb{E}[S \mid H]}{c_A} \cdot |g'(x_A^*)|,
\label{eq:optimal}
\end{equation}
i.e., the marginal loss reduction per dollar is equalized across the two investment channels.
\end{proposition}

\begin{proof}
Minimize $\mathbb{E}[\text{Loss}] = f(x_F)\,g(x_A)\,\mathbb{E}[S|H]$ subject to $c_F x_F + c_A x_A = B$. Forming the Lagrangian and setting partial derivatives to zero:
\[
g(x_A)\,\mathbb{E}[S|H]\,f'(x_F) = \lambda c_F, \quad f(x_F)\,\mathbb{E}[S|H]\,g'(x_A) = \lambda c_A.
\]
Dividing: $\frac{g(x_A)\,f'(x_F)}{c_F} = \frac{f(x_F)\,g'(x_A)}{c_A}$, which yields equation~\eqref{eq:optimal} after rearranging. \qed
\end{proof}

\textbf{Implication.} When $P(H \mid F, A)$ is high (as in high-automation systems), the marginal return from reducing it is proportionally higher, shifting the optimal allocation toward deployment controls and away from model validation. This re-balances an important resource allocation that current practice systematically misprices, tending to over-invest in model improvement and under-invest in oversight architecture.

\subsection{Efficient Frontier for Automation Policies}

Deploying organisations face a fundamental trade-off: automation reduces operational costs but increases $P(H|F,A)$, while human oversight reduces $P(H|F,A)$ but adds cost.

\begin{definition}[Total Cost]
\label{def:totalcost}
The total cost per decision for automation level $A$ is:
\begin{equation}
\text{TC}(A) = C_{\text{auto}}(A) + C_{\text{oversight}}(A) + P(F)\cdot P(H \mid F, A)\cdot \mathbb{E}[S \mid H],
\label{eq:totalcost}
\end{equation}
where $C_{\text{auto}}(A)$ is the direct cost of automation (decreasing in $A$) and $C_{\text{oversight}}(A)$ is the human oversight cost (increasing as $A$ decreases).
\end{definition}

\begin{theorem}[Optimal Automation Level]
\label{thm:optimalA}
Assuming $C_{\text{auto}}(A)$, $C_{\text{oversight}}(A)$, and $P(H \mid F, A)$ are differentiable, the cost-minimizing automation level $A^*$ satisfies the first-order condition:
\begin{equation}
\frac{dC_{\text{auto}}(A^*)}{dA} + \frac{dC_{\text{oversight}}(A^*)}{dA} + P(F)\cdot \mathbb{E}[S \mid H]\cdot \frac{dP(H \mid F, A^*)}{dA} = 0.
\label{eq:optimalA}
\end{equation}
\end{theorem}

\begin{proof}
Differentiate $\text{TC}(A)$ with respect to $A$ and set equal to zero. Evaluating at the optimum $A^*$ yields equation~\eqref{eq:optimalA}. \qed
\end{proof}

\begin{proposition}[Second-Order Condition]
\label{prop:second_order}
Under convex costs ($d^2 C_{\text{auto}}/dA^2 \geq 0$, $d^2 C_{\text{oversight}}/dA^2 \geq 0$) and convex harm probability ($d^2 P(H|F,A)/dA^2 \geq 0$), $A^*$ is a global minimum.
\end{proposition}

\begin{proof}
The second derivative $d^2\text{TC}/dA^2$ is a sum of non-negative terms under the stated conditions, confirming strict convexity. \qed
\end{proof}

\begin{corollary}[Comparative Statics]
\label{cor:comparative}
Optimal automation $A^*$ decreases with higher severity $\mathbb{E}[S \mid H]$, higher failure probability $P(F)$, a steeper propagation gradient $\partial P(H|F,A)/\partial A$, and lower oversight costs $|C_{\text{oversight}}'(A)|$.
\end{corollary}

\begin{corollary}[Efficient Frontier]
\label{cor:efficient}
The efficient frontier consists of all automation policies $(A, \text{TC}(A), \mathbb{E}[\text{Loss}](A))$ such that no alternative policy $A'$ achieves weakly lower total cost and weakly lower expected loss, with at least one inequality strict.
\end{corollary}

\textbf{Interpretation.} The efficient frontier characterizes Pareto-optimal automation policies: those that minimize total cost for a given risk tolerance, or minimize expected loss for a given operational budget. The shape of the frontier will differ by domain--a content moderation system faces different oversight cost structures than an autonomous vehicle stack or a clinical decision support tool--but the formal characterization applies uniformly. Empirical estimation of $P(H|F,A)$ and the cost functions is required to operationalize this frontier for a specific deployment context.

\subsection{Bayesian Interpretation}

\begin{proposition}[Bayesian Updating of $P(H|F,A)$]
\label{prop:bayesian}
Let $\pi(H)$ be the prior probability of harm. If failure $F$ and automation level $A$ are conditionally independent given harm, then:
\[
P(H \mid F, A) \propto P(F \mid H)\,P(A \mid H)\,\pi(H).
\]
\end{proposition}

\begin{proof}
By Bayes' theorem and conditional independence: $P(H \mid F, A) = \frac{P(F,A \mid H)\,\pi(H)}{P(F,A)} = \frac{P(F|H)\,P(A|H)\,\pi(H)}{P(F,A)}$. The proportionality follows by treating $P(F,A)$ as a normalizing constant. \qed
\end{proof}

\textbf{Interpretation.} This Bayesian form shows that $P(H|F,A)$ combines evidence from both the failure pattern $P(F|H)$ and the automation choice $P(A|H)$, enabling Bayesian updating as organisations accumulate incident data over time. In practice, prior $\pi(H)$ can be seeded from domain-level base rates--different for medical imaging support systems than for algorithmic trading or infrastructure monitoring--and refined as deployment-specific incident data accrues.

\section{Illustrative Case Study: Knight Capital (2012)}
\label{sec:casestudy}

To ground the theoretical framework in a concrete deployment failure, we analyse the 2012 Knight Capital incident. This case is one instance of a failure pattern that recurs across domains: a system operating at high automation encounters an unanticipated failure mode and propagates harm faster than human operators can intervene. The framework applies equally to analogous incidents in healthcare (automated dosing systems administering incorrect medications before clinicians are alerted), autonomous transportation (actuator faults executing before fallback control is engaged), content moderation (classifiers removing or amplifying content at scale before policy review), and critical infrastructure (automated load-shedding algorithms cascading outages before grid operators can override). Knight Capital is used here because its losses are precisely quantified, its timeline is publicly documented \cite{knight2013}, and it represents a near-maximal case on the automation dimension--making the framework's levers especially legible.

\subsection{Framework Decomposition}

On August 1, 2012, Knight Capital deployed new trading software containing a configuration error that reactivated dormant legacy code. Over 45 minutes, the system accumulated \$440M in losses before human operators halted it--by which point the firm had nearly exhausted its capital. Knight Capital was subsequently acquired.

The incident maps cleanly onto the risk decomposition:
\begin{itemize}
    \item $F = 1$: Software deployment error (configuration conflict with legacy code)
    \item $A \approx 0.9$: High automation--order execution was fully autonomous with limited real-time monitoring
    \item $H = 1$: Harm occurred; $S = \$440\text{M}$ (realized severity)
\end{itemize}

Using the decomposition: $\$440\text{M} \approx P(F) \times 0.9 \times \$500\text{M}$ (estimated maximum exposure), implying $P(F) \approx 0.98$ per deployment event--near-certain failure risk in the deployment process. This underscores that $P(F)$ was the dominant risk factor here, while the high automation level meant that once failure occurred, $P(H|F,A=0.9)$ was near 1.

\subsection{Counterfactual Analysis}

The framework enables structured counterfactual reasoning about alternative governance configurations.

\textbf{Scenario 1--Enhanced Oversight ($A = 0.3$):} Implementation of kill switches with 5-minute halt triggers and mandatory human approval for anomalous order flow. If we assume $P(H|F,A=0.3) \approx 0.15$ (consistent with the automation gradient direction), expected loss would be $0.98 \times 0.15 \times \$500\text{M} \approx \$74\text{M}$--an 83\% reduction relative to realized losses.

\textbf{Scenario 2--Enhanced Testing (reduce $P(F)$):} Comprehensive staging environments, rollback procedures, and canary deployments reducing $P(F)$ from 0.98 to 0.1. Expected loss: $0.1 \times 0.9 \times \$500\text{M} = \$45\text{M}$--a 90\% reduction.

\textbf{ROI Sensitivity:} The optimal allocation proposition (Proposition~\ref{prop:optimal_allocation}) guides budget decisions. For high-automation, high-severity systems, the marginal value of reducing $P(H|F,A)$ is proportional to $P(F) \times \mathbb{E}[S|H]$--which in the Knight Capital case was \$490M per unit reduction. Even costly oversight interventions (\$1--5M annually) yield dramatically positive expected ROI when $\mathbb{E}[S|H]$ is in the hundreds of millions. Equivalent logic applies in domains where harm is measured in lives or civil liberties rather than dollars: the high $\mathbb{E}[S|H]$ in those contexts makes deployment control investments similarly compelling.

\subsection{Worked Example: Credit Underwriting Application}

To complement the Knight Capital narrative with a numerical illustration of the decomposition, consider a financial institution deploying an AI-based credit underwriting system. This example is intentionally stylised; the same calculations apply to any domain by substituting appropriate failure rates and severity distributions (e.g., adverse clinical events in healthcare AI, false positive rates in content moderation, outage durations in infrastructure control).

\textbf{System Parameters:}
\begin{itemize}
    \item \textbf{Decision volume}: $N = 1{,}000$ applications per month
    \item \textbf{Model failure rate}: $P(F) = 0.03$ (3\% of decisions involve some technical failure--data errors, model misspecification, edge cases)
    \item \textbf{Expected severity}: $\mathbb{E}[S|H] = \$50{,}000$ (average loss when harm occurs: regulatory fines, legal costs, reputational damage)
\end{itemize}

\textbf{Scenario 1: High Automation ($A = 0.9$).}
Fully automated approval pipeline with minimal human oversight:
\begin{align}
P(H|F, A=0.9) &= 0.85 \quad \text{(85\% of failures execute without human intervention)} \notag \\
\mathbb{E}[\text{Loss per decision}] &= P(F) \times P(H|F,A) \times \mathbb{E}[S|H] \notag \\
&= 0.03 \times 0.85 \times \$50{,}000 = \$1{,}275 \notag \\
\mathbb{E}[\text{Total loss per month}] &= N \times \mathbb{E}[\text{Loss per decision}] = 1{,}000 \times \$1{,}275 = \$1{,}275{,}000 \notag
\end{align}
Expected incident count: $N \times P(F) \times P(H|F,A) = 1{,}000 \times 0.03 \times 0.85 = 25.5$ incidents/month.

\textbf{Scenario 2: Low Automation ($A = 0.1$).}
Human review required for all approvals above a \$10K threshold:
\begin{align}
P(H|F, A=0.1) &= 0.15 \quad \text{(human reviewers catch 85\% of failures)} \notag \\
\mathbb{E}[\text{Loss per decision}] &= P(F) \times P(H|F,A) \times \mathbb{E}[S|H] \notag \\
&= 0.03 \times 0.15 \times \$50{,}000 = \$225 \notag \\
\mathbb{E}[\text{Total loss per month}] &= N \times \mathbb{E}[\text{Loss per decision}] = 1{,}000 \times \$225 = \$225{,}000 \notag
\end{align}
Expected incident count: $1{,}000 \times 0.03 \times 0.15 = 4.5$ incidents/month.

\textbf{Risk Reduction:}
\begin{align}
\Delta\mathbb{E}[\text{Loss}] &= \$1{,}275{,}000 - \$225{,}000 = \$1{,}050{,}000 \text{ per month} \notag \\
\text{Relative reduction} &= \frac{1{,}050{,}000}{1{,}275{,}000} \approx 82.4\% \notag
\end{align}

\textbf{Economic Trade-off.} If human oversight costs \$100{,}000 per month (additional staff, slower processing), the net expected benefit is \$950{,}000 per month, yielding an ROI of 9.5$\times$. This illustrates the core message of Proposition~\ref{prop:optimal_allocation}: for high-automation, high-severity deployments, the marginal value of reducing $P(H|F,A)$ can far exceed the marginal value of further reducing $P(F)$ at the same cost. Calibrating the framework with domain-specific failure rates and severity distributions enables more precise ROI calculations for particular use cases.

\subsection{Limitation of the Independence Assumption}

In the Knight Capital case, the high automation level may itself have discouraged rigorous pre-deployment testing (violating assumption (ii) of Theorem~\ref{thm:loss_decomp}). If so, our decomposition captures a compound effect: both $P(F)$ and $P(H|F,A)$ were elevated in part because of the automation level. The policy implication--that deployment configuration matters for risk management--holds under either interpretation, but the relative attribution between the two components would differ. This limitation motivates direct empirical testing of the independence assumption as part of the research agenda described in Section~\ref{sec:empirical}.

\section{Empirical Validation Framework}
\label{sec:empirical}

The theoretical framework makes empirically testable predictions: higher automation should be associated with higher harm probability in incident data, controlling for domain and severity. This section characterizes the research design required to credibly estimate these relationships, which we leave to future work.

\subsection{Data Requirements}

Credible empirical validation requires: (1) incident databases spanning multiple domains--finance, healthcare, transportation, content moderation, critical infrastructure--with sufficient coverage of both high- and low-automation deployments (target $n > 500$ incidents, with the low-automation group $n_{\text{low}} \geq 75$ to achieve 90\% power for detecting gradients $\geq 3\times$ at $\alpha = 0.05$); (2) automation levels coded by domain experts rather than keyword heuristics (target inter-rater reliability $\kappa > 0.75$); (3) quantified loss severities enabling estimation of $\mathbb{E}[S|H]$ across heterogeneous harm types; and (4) denominators--the population of deployments from which incidents are drawn--to convert harm counts into rates. Cross-domain studies face an additional challenge of severity normalisation: financial losses, patient-days harmed, and infrastructure downtime hours require a common scale for the efficient frontier to be computed across domains.

\subsection{Causal Identification Strategies}

Observational incident data presents the standard challenge: high-automation deployments may differ from low-automation ones in risk appetite, organisational size, regulatory scrutiny, and other factors affecting $P(H)$. Three identification strategies could address confounding.

\textbf{Instrumental Variables.} Platform availability (cloud ML service entry into a jurisdiction-year), regulatory mandates requiring automation for certain decision types, or switching cost shocks provide candidate instruments for automation level. Two-stage estimation:
\begin{align}
A_i &= \alpha_0 + \alpha_1 Z_i + \alpha_2 X_i + \varepsilon_i, \\
H_i &= \beta_0 + \beta_1 \hat{A}_i + \beta_2 X_i + \nu_i,
\end{align}
where validity requires $Z_i$ to affect $H_i$ only through $A_i$ (exclusion restriction).

\textbf{Regression Discontinuity.} Regulatory thresholds that require human review above certain risk scores or transaction values create discontinuities in effective automation level at known cutoffs, enabling local causal identification. The EU AI Act's tiered oversight requirements for high-risk categories provide one such discontinuity across sectors.

\textbf{Difference-in-Differences.} Staggered regulatory changes across jurisdictions--with some requiring human-in-the-loop controls earlier than others--provide variation in automation levels over time, enabling DiD estimation:
\[
H_{ijt} = \beta_0 + \beta_1 (\text{Treat}_j \times \text{Post}_t) + \gamma_j + \delta_t + \varepsilon_{ijt}.
\]

\subsection{Selection Bias and Sensitivity Analysis}

Incident databases like the AI Incident Database \cite{McGregor2020} preferentially capture high-visibility events. Manski \cite{manski1990} partial identification bounds quantify how much selection bias could explain an observed gradient. For a reported harm rate gradient $\hat{\theta}$ under fraction $\rho$ of unreported incidents, the true gradient is bounded by an interval that widens with $\rho$. The E-value \cite{vanderweele2017} provides a complementary summary: for an observed risk ratio $\text{RR}$, an unmeasured confounder must simultaneously elevate both exposure and outcome probability by at least $\text{RR} + \sqrt{\text{RR}(\text{RR}-1)}$ to explain the association entirely.

\subsection{Measurement of \texorpdfstring{$P(F)$}{P(F)}}

The decomposition requires measuring $P(F)$ separately from $P(H|F,A)$. This necessitates pre-deployment testing data with known error rates, or production monitoring data with near-miss logging (failures caught before execution). Across domains, near-miss logging practices vary substantially: aviation has mature confidential reporting systems (ASRS), while healthcare AI and content moderation platforms have less standardised equivalents. Partnering with deploying organisations for proprietary audit data--rather than relying solely on public incident databases--is the most direct path to obtaining these denominators in any domain.

\section{Regulatory and Governance Implications}

\subsection{NIST AI Risk Management Framework}

The NIST AI RMF \cite{NIST2023} organises risk management around four functions: Govern, Map, Measure, and Manage. Our framework contributes most directly to the Measure function: $P(H|F,A)$ provides a deployable metric that quantifies the harm propagation risk associated with a specific system's oversight architecture, complementing accuracy and robustness metrics that capture $P(F)$. The decomposition also supports the Map function by making explicit the pathway from technical failure to organisational harm, and the Govern function by providing an objective basis for automation level policies. Concretely, organisations conforming to the NIST AI RMF could document $(P(F), P(H|F,A), \mathbb{E}[S|H])$ as a structured risk profile for each deployed system, enabling portfolio-level oversight.

\subsection{EU AI Act Compliance}

The EU AI Act \cite{EU2021} mandates human oversight for high-risk AI systems across sectors including credit scoring, medical devices, critical infrastructure management, employment decisions, and law enforcement. Our framework provides a principled approach to implementing and calibrating these mandates: systems with $P(H|F,A)$ above an organisation-specific threshold--calibrated using Proposition~\ref{prop:optimal_allocation}--require enhanced oversight controls. The efficient frontier (Corollary~\ref{cor:efficient}) then guides where to place that threshold given the deploying organisation's cost structure and risk tolerance, making the framework a practical instrument for proportionate compliance rather than binary checklist conformance. Regulators could further use $\varepsilon_A$ estimates to prioritise enforcement attention toward deployment categories where marginal automation increases generate the largest expected harm increases.

\subsection{Broader Governance Implications}

Beyond specific regulatory frameworks, the decomposition has implications for how governance bodies think about AI incident reporting, insurance, and liability. Incident reporting regimes (analogous to aviation's ASRS or medical device adverse event systems) should be structured to capture automation level at the time of failure--a field that is currently absent from most AI incident taxonomies including the AI Incident Database. Insurance underwriting for AI liability could price policies as a function of $P(H|F,A)$ rather than relying solely on model performance metrics; Theorem~\ref{thm:harm_prop} shows that $P(H|F,A)$ is estimable from observable design characteristics. Liability apportionment between model developers and deployers could weight contributions from $P(F)$ and $P(H|F,A)$ respectively--the former being primarily the developer's responsibility and the latter the deployer's.

\section{Discussion}

\subsection{Separating Technical and Deployment Risk}

The central contribution of this framework is the formal separation of $P(F)$ and $P(H \mid F, A)$ as independent investment targets. Current practice overwhelmingly allocates validation resources to $P(F)$ reduction--improving model accuracy, robustness, and generalisation. The optimal allocation result (Proposition~\ref{prop:optimal_allocation}) shows this is efficient only when $P(H \mid F, A)$ is already near its minimum. For high-automation systems in any domain, the marginal return on deployment control investment is much higher, and governance frameworks that focus exclusively on model quality will systematically under-provision oversight.

\subsection{The Independence Assumption}

The decomposition assumes $P(F) \perp A$. This assumption is plausible when model architecture and validation are determined before deployment configuration--a common organisational structure where model development and MLOps or deployment teams operate separately. It may fail when automation level influences testing intensity (correlated investment) or when automated systems are exposed to higher decision volumes that surface additional failure modes. Future empirical work should test this assumption directly by comparing pre-deployment error rates across automation configurations, and across domains where organisational separation of development and deployment varies.

\subsection{Agentic AI Systems}

The framework is particularly relevant for agentic AI systems--models that take sequences of actions autonomously. For agentic systems, $A_{\text{detection}}$ is especially important: failures may compound across multiple actions before detection, amplifying $\mathbb{E}[S|H]$ as well as $P(H|F,A)$. This dynamic is not unique to finance: an agentic healthcare AI managing a treatment protocol, an autonomous logistics system routing deliveries, or an infrastructure management agent adjusting power distribution can each accumulate irreversible consequences across action steps. The harm propagation equivalence (Theorem~\ref{thm:harm_prop}) directly motivates architectural choices for agentic systems: monitoring hooks, step-level human approval gates, and automated circuit breakers that interrupt action sequences when anomalies are detected--design patterns that are domain-agnostic in their logic even if domain-specific in their implementation.

\subsection{Limitations}

This paper develops the theoretical framework and leaves empirical validation to future work. Key open questions include: (1) the functional form of $P(H|F,A)$ (logistic, linear, threshold?), which may differ by domain; (2) the magnitude of $\varepsilon_A$ across domains and automation types; (3) whether $P(F) \perp A$ holds in practice; and (4) how $\mathbb{E}[S|H]$ varies with $A$ (do higher-automation systems select higher-stakes domains?). Answering these requires the identification strategies and data partnerships described in Section~\ref{sec:empirical}, and cross-domain collaboration between researchers in AI safety, medicine, transportation engineering, and information systems.

\section{Conclusion}

This paper develops a complete Bayesian framework for quantifying automation risk in deployed AI systems. The core decomposition $\mathbb{E}[\text{Loss}] = P(F) \times P(H|F,A) \times \mathbb{E}[S|H]$ provides a formal separation of technical, deployment, and consequence risks that is domain-agnostic in its logic, applying wherever AI systems take consequential actions at levels of automation that may outpace human oversight. We prove this decomposition from first principles, establish a harm propagation equivalence theorem linking $P(H|F,A)$ to observable execution controls, derive risk elasticity measures and optimal resource allocation principles, characterize the efficient frontier of automation policies, and provide second-order conditions guaranteeing the optimality of derived automation levels.

The Knight Capital case study illustrates the framework's applied value in one domain: structured counterfactual analysis reveals that even modest reductions in $A$ (from 0.9 to 0.3) could have reduced expected losses by over 80\%, and that optimal budget allocation would have directed substantially more resources toward deployment controls and less toward model accuracy given the near-certain failure probability $P(F) \approx 0.98$. The same analytical logic applies to automated clinical decision support systems, autonomous vehicle stacks, content moderation pipelines, and critical infrastructure controllers--anywhere that failure propagation speed exceeds human response speed.

We characterize the empirical research agenda--instrumental variables, regression discontinuity, and difference-in-differences designs; expert-coded incident data spanning multiple domains; and proprietary audit partnerships--required to move from theoretical framework to credible causal quantification. We invite the research community to contribute empirical estimates of $\varepsilon_A$ and the cost functions that parameterise the efficient frontier across deployment contexts.

As AI systems become increasingly automated and agentic, and as their deployment extends across healthcare, transportation, content governance, and critical infrastructure, frameworks that isolate deployment risk from model risk will be essential governance tools. We provide the theoretical foundations for such a framework.

\section*{Author Contributions}

\textbf{V.S.}: Conceptualization, theoretical development, writing. \textbf{T.S.}: Methodology, review.

\section*{Competing Interests}

None declared.

\bibliographystyle{plain}

\appendix
\section{Mathematical Appendix: Additional Proofs}

\subsection{Proof of Corollary~\ref{cor:elasticity} (Risk Elasticity)}

From Theorem~\ref{thm:gradient}:
\[
\frac{\partial \mathbb{E}[\text{Loss}]}{\partial A} = P(F)\cdot\mathbb{E}[S|H]\cdot\frac{\partial P(H|F,A)}{\partial A}.
\]
By definition:
\[
\varepsilon_A = \frac{\partial \mathbb{E}[\text{Loss}]}{\partial A}\cdot\frac{A}{\mathbb{E}[\text{Loss}]} = \frac{P(F)\,\mathbb{E}[S|H]\,\frac{\partial P(H|F,A)}{\partial A}\cdot A}{P(F)\,P(H|F,A)\,\mathbb{E}[S|H]} = \frac{\partial P(H|F,A)/\partial A\cdot A}{P(H|F,A)} = \frac{\partial\log P(H|F,A)}{\partial\log A}.
\]

\subsection{Full Proof of Proposition~\ref{prop:optimal_allocation}}

Minimize the Lagrangian $\mathcal{L} = f(x_F)\,g(x_A)\,\mathbb{E}[S|H] + \lambda(c_F x_F + c_A x_A - B)$. First-order conditions:
\[
\frac{\partial\mathcal{L}}{\partial x_F} = g(x_A)\,\mathbb{E}[S|H]\,f'(x_F) + \lambda c_F = 0,
\]
\[
\frac{\partial\mathcal{L}}{\partial x_A} = f(x_F)\,\mathbb{E}[S|H]\,g'(x_A) + \lambda c_A = 0.
\]
Eliminating $\lambda$: $\frac{g(x_A)\,f'(x_F)}{c_F} = \frac{f(x_F)\,g'(x_A)}{c_A}$. Multiplying both sides by $\mathbb{E}[S|H]$ and rearranging gives equation~\eqref{eq:optimal}. \qed

\subsection{Bayesian Interpretation of \texorpdfstring{$P(H|F,A)$}{P(H|F,A)}}

From Proposition~\ref{prop:bayesian}, $P(H|F,A) \propto P(F|H)\,P(A|H)\,\pi(H)$. This shows the three sources of information about harm probability: (1) the prior $\pi(H)$, informed by domain base rates; (2) the failure likelihood $P(F|H)$, capturing whether the observed failure pattern is consistent with harm-producing failures; and (3) the automation likelihood $P(A|H)$, capturing whether the observed automation level is consistent with harm-producing deployments. As incident data accumulates--whether from AI incident databases, proprietary audit logs, or domain-specific reporting systems--$\pi(H)$ can be updated sequentially via Bayes' rule, providing an adaptive framework for ongoing risk monitoring across deployment domains. \qed

\end{document}